\documentclass[]{interact}

\usepackage[numbers]{natbib}
\usepackage{hyperref} 

\usepackage{graphicx}
\usepackage{float}
\usepackage{subcaption}
\usepackage{caption}
\usepackage{multirow}
\usepackage[linesnumbered,algoruled,boxed,lined]{algorithm2e}
\SetKwInOut{Parameter}{Parameters}
\usepackage{indentfirst}
\usepackage{epstopdf}

\renewcommand\bibfont{\fontsize{10}{12}\selectfont}

\theoremstyle{plain}
\newtheorem{theorem}{Theorem}[section]
\newtheorem{lemma}[theorem]{Lemma}

\newtheorem{prop}[theorem]{Proposition}

\theoremstyle{definition}
\newtheorem{defn}[theorem]{Definition}

\theoremstyle{remark}

\begin{document}

\articletype{ORIGINAL RESEARCH ARTICLE}

\title{Adaptive Resources Allocation CUSUM for Binomial Count Data Monitoring  with Application to COVID-19 Hotspot Detection}

\author{
\name{Jiuyun Hu\textsuperscript{a}, Yajun Mei\textsuperscript{b}, Sarah Holte\textsuperscript{c}, Hao Yan\textsuperscript{a}\thanks{CONTACT A.~N. Author. Email: haoyan@asu.edu}}
\affil{\textsuperscript{a}School of Computing and Augmented Intelligence, Arizona
State University, Tempe, AZ, USA; \textsuperscript{c}School of Industrial and Systems Engineering, Georgia Institute of Technology, Atlanta,
GA, USA; \textsuperscript{c}Division of Public Health Sciences, Fred Hutchinson Cancer Research Center, Seattle, WA, USA}
}

\maketitle

\begin{abstract}
In this paper, we present an efficient statistical method (denoted as "Adaptive Resources Allocation CUSUM") to robustly and efficiently detect the hotspot with limited sampling resources. Our main idea is to combine the multi-arm bandit (MAB) and change-point detection methods to balance the exploration and exploitation of resource allocation for hotspot detection. Further, a Bayesian weighted update is used to update the posterior distribution of the infection rate. 
Then, the upper confidence bound (UCB) is used for resource allocation and planning. Finally, CUSUM monitoring statistics to detect the change point as well as the change location. For performance evaluation, we compare the performance of the proposed method with several benchmark methods in the literature and showed the proposed algorithm is able to achieve a lower detection delay and higher detection precision. Finally, this method is applied to hotspot detection in a real case study of county-level daily positive COVID-19 cases in Washington State WA) and demonstrates the effectiveness with very limited distributed samples. 

\end{abstract}



\section{Introduction \label{sec:Introduction}}
Nowadays, streaming data from multiple data sources or streams have become more and more common in public health surveillance. Rapid \textit{hotspot} detection for the streaming data across different spatial regions over time is important and can provide valuable information to the decision-makers to mitigate the risk as soon as possible. In this paper, we define the \textit{hotspot} as the structured outliers that are persistent after a certain time point. Besides, we also assume that full observation is not possible due to the limited sampling resources. 

A motivating example for this research is to monitor the number of COVID-19 confirmed cases for different spatial regions in the United States. Monitoring the entire population is not possible due to the limited testing resources or labels. Here, we would like to focus on adaptively distributing the testing resources such that the hotspot can be detected as soon as possible. Please see Section \ref{sec:Motivation} for a more detailed description of the COVID-19 monitoring problem.

Generally speaking, hotspot detection in spatio-temporal data can be treated as a change-point detection problem. Take the COVID-19 monitoring as an example. There are three types of challenges for the hotspot detection problem: 1) High-dimensionality: the number of sites or locations to be monitored is often quite large. 2) Spatial sparsity of the hotspot: the affected counties are often quite sparse in the high-dimensional space. It implies that the counties undergoing significant changes in the number of confirmed cases should be small at certain time points. 3) Temporal consistency: the hotspots should last for a reasonably long period of time.

Besides the aforementioned challenges, there are also additional challenges in sequential change detection related to the limited sensing resource as follows: 1) Lack of access to the fully observational data: existing research for high dimensional sequential sparse change detection focuses on a fully observable process, i.e., at each sampling time point, all the variables can be observed for analysis. However, in reality, it is infeasible to acquire full measurements for the entire population. 2) Partial sampling with count data observation: existing research often focuses on the continuous observational data and often assumes that a location is either observed or not observed. However, in many examples, such as the COVID-19 monitoring case, the decision variables are not binary, which implies that you can distribute more test kits to a particular region to obtain a more accurate estimation of the percentage of the infected person. In such cases, the distribution of the sampling resources is much more complicated than simply deciding whether a region should be observed or not. 3) Adaptive decision-making to balance the exploration and exploitation: On one hand, we would distribute more tests in the particular region that has been detected with a high infection rate. On the other hand, other unobserved regions may be potentially at risk as well, and some testing resources should be distributed in these regions as well. 
Overall, the trade-off between exploration and exploitation is a fundamental challenge in multi-arm bandit (MAB) and reinforcement learning. Recently, Such a balance is also an important aspect to be considered in adaptive resource allocation problems for change point detection, e.g., in \cite{xu2021optimum, zhang2019partially, guo2020partially, zhang2020bandit}. Take Figure \ref{figure-US infection rate} as an example, the infection rates in Florida and Louisiana were higher than in other states on Sep 13, 2020. However, the resource allocation algorithm should distribute enough samples randomly to all states to discover such a fact (i.e., exploration) before distributing more tests to these two states specifically (i.e., exploitation) for final confirmation and change point detection.

\begin{figure}
        \centering
        \includegraphics[scale=.6]{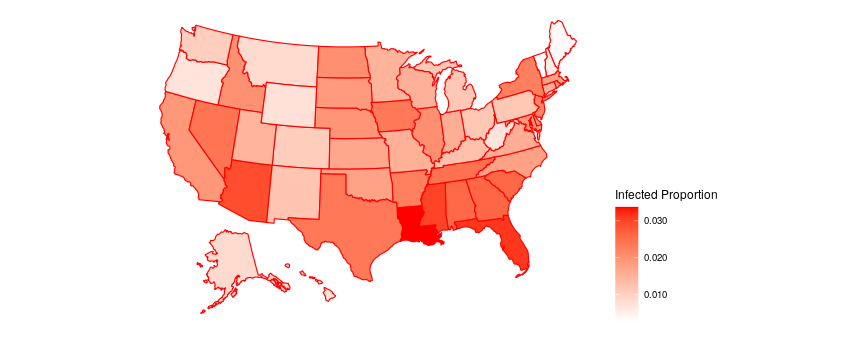}
        \caption{Infection Rate on Sep 13, 2020 in the United States}
        \label{figure-US infection rate}
    \end{figure}


This paper aims to develop an adaptive resource allocation method to automatically distribute the number of testing resources into multiple regions for the quickest change point and hotspot detection. More specifically, we will focus on the case where the data follows a binomial distribution, where the total number of tests for each region $i$ (i.e., $n_i$) depends on the number of testing resources distributed to region $i$. Overall, we propose to borrow the idea from the Upper Confidence Bound (UCB) method in solving the MAB problem to automatically distribute the testing resources to balance the exploration and exploitation. Bayesian statistics are used to update the estimation of the infection rate online with the consideration of uncertainty. A "reward" function is proposed based on the CUSUM statistics such that the change point can be identified as soon as possible. 

Intuitively, the algorithm should work adaptively. Typically, the testing resources should be distributed equally to different regions in the early stages (i.e., exploration). If some regions with potentially high infection rates are observed, more testing resources should be distributed in these regions for confirmation (i.e., exploitation). 
Furthermore, for resource allocation, we will borrow the idea from MAB. MAB aims to sequentially allocate a limited set of resources between competing "arms" to maximize their expected gain, where each arm’s reward function is unknown. MAB provides a principled way to balance exploration and exploitation. To apply MAB for change point detection under the sampling constraint, we propose to use MAB to combine the CUSUM statistics into the reward function in the MAB problem. However, different from \cite{zhang2020bandit, guo2020partially}, the "arm" does not refer to whether a single sensor should be observed or not, but the combinations of the number of testing resources (e.g., testing kits) should be distributed to each particular region. Finally, an integer programming algorithm is proposed to optimize such reward function for resource allocation.

The rest of the paper is organized into the following parts. In Section \ref{sec:Motivation}, we describe the motivation example of the problem in detail.
In Section \ref{sec:Literature-Review}, we review some related methods in change point detection. In Section \ref{sec:Proposed-Methodology}, we introduce the formulation and algorithm of the proposed method in detail. In Section \ref{sec:Simulation}, we use the simulation to study the behavior and some properties of the algorithm. A threshold is also generated to be applied to the case study. In Section \ref{sec:Case Study}, we study the case of county-level daily positive cases in Washington State (WA). Section \ref{sec:Conclusion} summarises the result and contributions of this study.

\section{Motivation Example}\label{sec:Motivation}
Since the initial outbreak of the novel coronavirus in early January 2020, the COVID-19 pandemic has rapidly spread across the world and brought enormous disruption to the economy and society. Various emergency measures, such as social distancing, school closures, and economic shutdowns, have been taken by different countries to control the first wave of the pandemic. To balance saving lives and the economy, a condition-based phased approach has since been adopted in the USA, where the strictness of public health measures is set to adapt according to the present epidemic condition. The success of such a phased approach critically depends on the accurate assessment of the current and near-future status of the pandemic.

Much recent research has been focusing on modeling and prediction of the spread of the COVID-19 case report data. For example, \citet{JIANG2020} studied the pandemic from a statistical point of view. The piece-wise linear trend model has been applied to model the positive case growth and the self-normalization-based method was applied to detect the change points and predict cases in the future. 
\citet{tariq2020real} monitored pandemic data in Singapore and estimated the effective reproduction number $R_t$. They also took advantage of different local clusters and some international transmission. Despite various studies to forecast the positive cases and deaths, \citet{luo2020predictive} pointed out that it would be hard to predict the future of the pandemic. The direction of the mutant of a virus is random and, thus, hard to predict whether it will be more contagious or not. For example, the Delta and Omicron variant of COVID-19 caused a remarkable increase of positive cases in the US starting in June 2021 and Jan 2022. 

Despite the prediction and modeling of the COVID-19 spread, many recent works focus on developing a monitoring framework to detect the pandemic outbreak while taking advantage of a more scientific resource allocation method to distribute the test kits and detect the outbreak as soon as possible. 
\citet{scobie2021monitoring monitored and demonstrated the importance of the vaccine in terms of the reduction of the incidence rate ratios (IRRs).}
\citet{astley2021global} used the data from The University of Maryland Global COVID Trends and Impact
Survey (UMD-CTIS) in cooperation with Facebook to monitor the pandemic.
 In general, this algorithm can monitor the COVID-19-related daily spatio-temporal cases with limited testing resources, and in this paper, we call this problem the adaptive partially observed spatio-temporal hotspot detection problem. 
 However, most of these works assume that the data is passively collected but do not focus on how to adaptively distribute the testing resources. \citet{chatzimanolakis2020optimal} discussed the optimal resources allocation problem in a different stage to achieve maximum information gain. 
 
To better understand the COVID-19 status, different types of testing resource is typically distributed in different regions. Centers for Disease Control and Prevention (CDC) has classified the testing for COVID-19 into the following two categories: 1) \textbf{Diagnostic testing} is intended to identify current infection in individuals and is performed when a person has signs or symptoms consistent with COVID-19, or is asymptomatic, but has recently known or suspected exposure to COVID-19. 2) \textbf{Screening tests} are recommended for unvaccinated people to identify those who are asymptomatic and do not have known, suspected, or reported exposure to COVID-19. Screening helps to identify unknown cases so that measures can be taken to prevent further transmission.\cite{covid192022}

Overall, the screening test is very useful in testing employees in a workplace setting or universities, testing a person before or after travel, or randomly distributed test in some underdeveloped areas to identify unknown cases so that measures can be taken to prevent further transmission. In this paper, we will focus on hotspot detection for the case report data with limited screening test resources. For simplicity, we assume that the screening test is the only available source of information, and limited testing resource is available. 

In this paper, we will use the real COVID-19 test report data from Johns Hopkins University Center for Systems Science and Engineering (JHU CSSE) (\citet{dong2020interactive}). The dataset is available on \url{https://github.com/CSSEGISandData/COVID-19}. 
More specifically, we will use the confirmed COVID-19 cases from all 39 counties in WA. 
The source of daily positive cases in WA is the Department of Health (\url{https://www.doh.wa.gov/Emergencies/COVID19}). The time-series data is updated daily around 23:59 (UTC). We will use the data from Jan 23, 2020, to Sep 13, 2020, a total of 235 days, as an illustration. On each day, the confirmed cases are recorded in all counties in the United States. We aim to identify hotspots in such a region and see if we can discover the hotspot with a much smaller number of screening test resources.

In this study, we will focus on the distribution of the number of screening testing resources for efficient hotspot detection. With this technique, we are able to conduct a real-time online monitoring process and actively emphasize the regions potentially at risk. 
First, we showed the infection rates in some example counties in Figure \ref{Total-WA-cases}. From this figure, we can see that the positive rates of Yakima County started to rise around day 70, and increased dramatically. Furthermore, we have added a figure of the proportion of infected cases in some of the time points during the spread of the COVID-19 data in Figure \ref{fig:data-summary}. We can see that the infection rates on day 100 are not very high for all counties. The infection rates in some of the counties started to rise.

\begin{figure}
        \centering
        \includegraphics[width=0.7\linewidth]{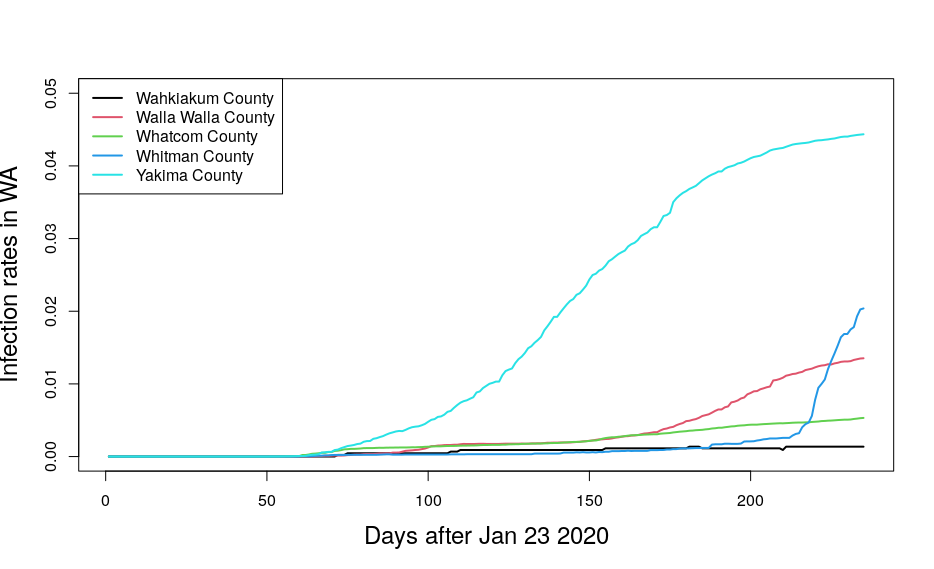}  
        \caption{Example of infection rates in WA}
        \label{Total-WA-cases}
\end{figure}

    
\begin{figure}[H]
        \begin{subfigure}{.5\textwidth}
            \centering
            \includegraphics[width=\linewidth]{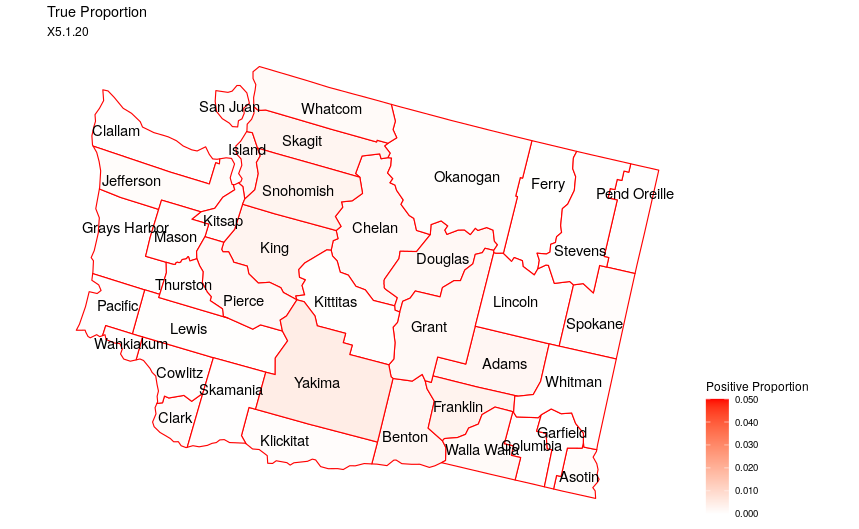}
            \caption{Infection proportion on day 100: May 1 2020}
            \label{Proportion-Example-100}
        \end{subfigure}
        \begin{subfigure}{.5\textwidth}
            \centering
            \includegraphics[width=\linewidth]{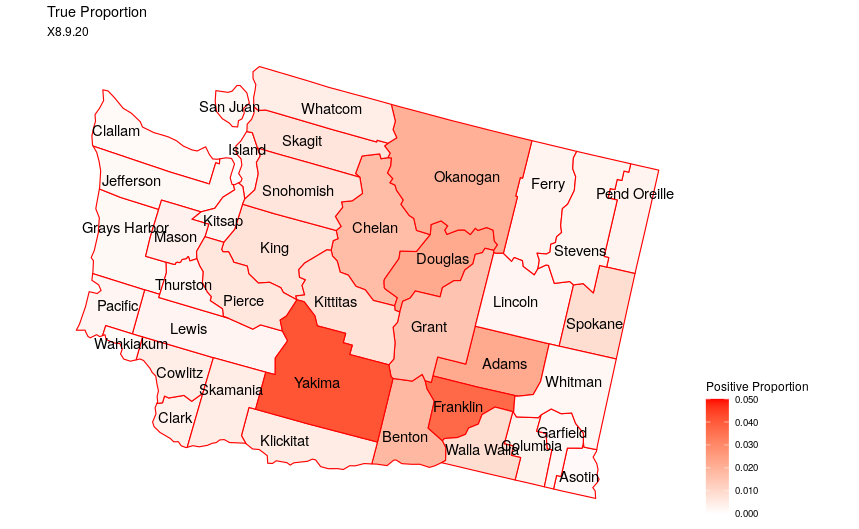}
            \caption{Infection proportion on day 200: Aug 9 2020}
            \label{Proportion-Example-200}
        \end{subfigure}
        
    \caption{Comparison of county-wise infection rate in WA on different dates}
    \label{fig:data-summary}
\end{figure}

\section{Literature Review \label{sec:Literature-Review}}
In this Section, we will present a literature review on the change point detection methodology that can be potentially applied to the proposed spatio-temporal problem, including change point detection on the high dimensional data stream, change point detection with sampling control, and change-point detection on count data.

\subsection{Change Point Detection for multiple data streams}

Originally, the change point detection is first proposed to detect some sudden changes in the \textit{univariate data stream} with the shortest detection delay. The CUSUM statistics is proposed by \citet{page1954continuous, page1955test} to minimize the supremum average detection delay  \cite{moustakides2009numerical}. Then, the Shiryaev\textendash Roberts procedure was introduced by \citet{shiryaev1961problem} in a Bayesian decision framework. Please see \cite{truong2020selective} for a review paper of the existing change point detection methods. However, as mentioned, these methods cannot be applied to problems with multiple data streams.

To extend the change-point detection method for \textit{multiple data streams}, \citet{mei2010efficient} extended the CUSUM statistics to multiple data streams by the summation of the CUSUM statistics for each data stream. \citet{xie2013sequential} compared the mixed procedure method they developed with some generalized log-likelihood ratio-based procedures to detect the change point in multiple but unknown data streams under the assumption that the data are normally distributed with known variance, and the change is a mean shift with unknown post-change mean. There are some works to extend these methods to multiple change-point detections with abrupt changes, such as \cite{zhang2021dynamic,avanesov2018change}, which considered the functional subspace learning and exhaustively search of the change point. \citet{enikeeva2019high} studied the change point detection when the number of dimensions and the length of each data stream goes to infinity based on the maximum of scan statistic and linear test statistic. However, these methods typically assume that the data follow a normal distribution or the data is fully observed. 

\subsection{Change Point Detection for Partially Observed Data with Sampling Control}
Another line of research focuses on detecting anomalies for the \textit{spatio-temporal data}. For example, \citet{das2019spatio} applied some tests to study the change point of winter maximum, minimum, and average temperature from $6$ meteorological stations in 102 year time period. \citet{chen2020s} proposed a new $S^3T$ statistic to detect change points in spatio-temporal data with mean shift or covariance structure change. \citet{zhao2019composite} developed a likelihood-based algorithm to solve the multiple change detection in non-stationary spatio-temporal processes by splitting the data into piecewise stationary processes. \citet{knoblauch2018spatio} developed a Bayesian online change-point detection algorithm to model the process between change points in spatio-temporal data. \citet{zhang2018multiple} studied the change point detection problem when there are system-inherent variations, and sensor feature differences and is able to identify the local change point in each cluster. Again, these methods focus on the fully observed data and do not provide a method to iteratively update the sampling resources. 

To generalize the change point detection for \textit{partially observed data}, some works focus on the change point detection by assuming that only partial data streams at each time point can be observed due to limited sampling resources. Such setting has been studied in social network monitoring \cite{farajtabar2015back}, biological network \cite{raue2009structural}, and manufacturing systems \cite{liu2015adaptive}. To deal with partially observed data, \citet{xie2012change} developed a change point detection method by combining subspace manifold learning, and multiscale analysis to solve the missing data problem. \citet{dubey2021online} studied the online change point detection problem in the network where the pattern of the missing data of the network is heterogeneous. \citet{corradin2022bayesian} studied multiple change-point detections in multivariate time series with missing values. The conditional distribution of each data point given the rest data is derived to handle the missing value problem. However, these methods do not actively control which subsets of data to measure, which hinders the detection performance. 

Recently, many existing works have focused on the change-point detection problem with sampling control, where the sensor(s) to be observed in each epoch are determined adaptively based on the current and historical observation. For example, \citet{liu2015adaptive} proposed to combine CUSUM statistics and a top-$r$ strategy to select the data streams to be observed for online change detection. \citet{xu2021multi} studied the problem when there is only one sensor to be observed at each time and an adaptive sampling strategy based on the CUSUM statistic and showed that it is asymptotically optimal in the sense of minimizing the detection delay. \citet{zhang2019partially} developed an online learning algorithm for anomaly detection for independent data streams using the combinatorial MAB approach. \citet{zhang2020bandit} studied the change point detection under sampling control by the combination of Thompson Sampling and a top-$r$ local Shiryaev-Robert statistic to raise a global alarm. 
\citet{guo2020partially} proposed a Bayesian spike-and-slab prior distribution to decompose the data into the smooth background and sparse anomaly and focus on identifying the anomaly by selecting the sensors to observe using Thompson sampling.
However, all aforementioned methods focus on change detection under the continuous observation data and assume each sensor is either fully observed or not, which cannot be used for more complex resource allocation problems for anomaly detection. 

\subsection{Change Point Detection on Counting Data}

Much existing research has been focusing on extending the change point detection method for \textit{count data or binary data}. For example, \citet{hinkley1970inference} studied the change point detection problem for binary variables based on the asymptotic distribution of the likelihood ratio test statistics. \citet{yu2013change} developed the CUSUM statistics to study the change point detection problem in the binomial thinning process and applied it to the 2001-2012 influenza data.
\citet{jiang2011weighted} proposed a weighted CUSUM control chart for monitoring Poisson processes with varying sample sizes. 
\citet{gut2002truncated} studied the change point detection problem in renewal counting data with the discrete-time frame, where the term `partially observed' meant the discrete-time frame. \citet{ellenberger2021exact} developed a binary segmentation method to solve the change point detection problem when the sample size of the binomial distribution is small. However, these methods typically assume that the sample size of the binomial distribution is pre-determined, which cannot be used to solve the resource allocation problem. 


\section{Proposed Methodology \label{sec:Proposed-Methodology}}

In this Section, we will explore the adaptive sampling method on discrete data with the binomial distribution.
In Section \ref{subsec:Problem-Formulation}, we will review the overall problem formulation. In Section
\ref{subsec:ProposedAlgorithm}, we will introduce the procedure of the proposed algorithm. In Section
\ref{subsec:Tuning-Parameter-Selection}, we will introduce the tuning parameter selection procedure for the proposed algorithm.


\subsection{Problem Formulation \label{subsec:Problem-Formulation} }
We assume that the data is $X_{k,t}$, where $t=1,2,\ldots$ is the time index and $k=1,2,\ldots,K$ is the region index. 
In the COVID-19 monitoring example, $X_{k,t}$ can be the number of  newly observed confirmed cases in region $k$ and time $t$. 

We further assume that all $X_{k,t}$ prior-change follows a binomial distribution $f_k^0$ as $\text{Bin}(c_{k,t},p)$, with the number of test $c_{k,t}$ and the probability $p$. Furthermore, we assume that after the change $t>t_0$, only one hotspot region (i.e., $k \in \mathcal{H}$) is impacted by the change. In contrast, we assume that the regions affected by the changes follow another post-change distribution $f_{k}^1$ where $f_{k}^1$ is another binomial distribution as  $\text{Bin}(c_{k,t},q)$, with the number of tests as  $c_{k,t}$ with the probability $q$, as shown in Eq. (\ref{eq: Binomial}). The superscript $0$ means in control, and the superscript $1$ means out of control.

\begin{equation}
X_{k,t} \sim \begin{cases}
			\text{Bin}(c_{k,t},p), & t\le t_0 \text{ or } t > t_0, k \notin \mathcal{H} \\
            \text{Bin}(c_{k,t},q), & t> t_0, k \in \mathcal{H} 
		 \end{cases} 
         \label{eq: Binomial}
\end{equation}
Finally, we assume that there are limited sampling resources. The constraint we assume is that limited testing kits $C$ are distributed to all the regions at each time $t$ with the constraint $\sum_k c_{k,t}=C$.
For simplicity, we assume that the actively distributed test kits are the only source of observation data. 
The objective of this paper is to sequentially distribute the testing resources $c_{k,t}$  for region $k$ at time $t$ for faster change point detection. 




\subsection{Proposed Adaptive Resources Allocation CUSUM for Binomial Count data (ARA-CUSUM) Method \label{subsec:ProposedAlgorithm}}

In this Section, we first introduce the overall step of the algorithm and then discuss three major components of the proposed methods as follows: monitoring statistics updates, planning for resource allocation, and change-point detection decision. The overall framework is shown in Figure \ref{fig: flow chart}.
The detailed steps are discussed as follows.
\begin{enumerate}
\item \textbf{Monitoring statistics update}: We propose to update the monitoring statistics at each time $t$ for the quickest change-point detection. This procedure will be discussed in detail in Section \ref{subsubsec:Statistics}.
\item \textbf{Planning for adaptive sampling}: We would like to distribute the testing resource for the next time point based on previous results by the UCB bandit algorithm. The detailed planning formulation and the optimization algorithm are discussed in Section \ref{subsubsec:Planning} and Section \ref{subsec:Optimization-algorithm}.
\item \textbf{Change point detection decision}: Finally, the alarm is raised when the largest CUSUM statistics for all regions exceeds some threshold. Finally, the corresponding region that triggers such change can also be used for hotspot identification. 
\end{enumerate}

\begin{figure}
    \centering
    \includegraphics[width=\linewidth]{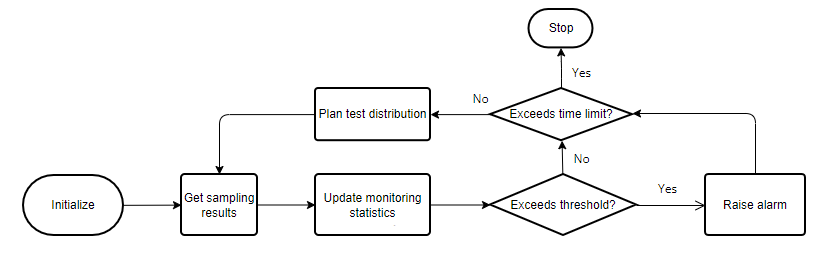}
    \caption{Flow chart for the algorithm}
    \label{fig: flow chart}
\end{figure}


\subsubsection{Monitoring Procedure \label{subsubsec:Statistics} }
In this Section, we will first introduce the definition of CUSUM statistics and then derive the monitoring procedure for online change detection. 
\begin{defn}
Suppose in a sequential data, $x_n$ is the score of $n$th observation. The CUSUM statistics \citet{page1954continuous} is defined by 
\begin{equation}
S_0=0, ~~ S_n = \max\{0, S_{n-1}+x_n\},(n\ge 1)
\end{equation}

\end{defn}
\citet{moustakides1986optimal} suggested that the CUSUM statistic is mini-max optimal in the sense of detection delay \cite{lorden1971procedures}. 
\begin{equation}
    D(T)=\sup_{\nu\ge 0}\text{ess}\sup \mathbb E_\nu [(T-\nu)^+|X_1,\ldots,X_\nu]
\end{equation}

with the constraint of in-control average run length $\mathbb E_\infty[T]\ge\gamma>0$.
In change point detection, the score $x_n$ for each observation can be the log-likelihood ratio, which is $x_n=\log\frac{f^1(X_n)}{f^0(X_n)}$, where $X_n$ is the $n$th observed data, $f^0$ is the pdf of $X_n$ before change and $f^1$ is the pdf of $X_n$ after change.

In the proposed formulation, $X_{k,t}$ is defined as the observed positive cases in time $t$ and region $k$. Inspired by the CUSUM statistics definition, the CUSUM statistics $W_{t}^{k}$ is formulated as: 
\begin{equation}
W_{k,t}=\max\{W_{k,t-1},0\}+\log\frac{f^1_{k}(X_{k,t})}{f^0_{k}(X_{k,t})},\label{eq-CUSUM}
\end{equation}
where $f^0_{k} \sim \text{Bin} (c_{k,t}, p)$ and $f^{1}_{k} \sim \text{Bin}(c_{k,t}, q)$ represent the distribution of null and alternative hypothesis, respectively. The initial value $W_{k,0}$ is set to be $0$ for all $k$. Following the CUSUM statistics definition, we define the CUSUM statistics $W_{k,t}$ in the proposed formulation for each county $k$ and time $t$ in Proposition \ref{prop-CUSUM}.
\begin{prop}
\label{prop-CUSUM} Suppose that at time $t$, the number of tests distributed at region $k$ is $c_{k,t}$,
the number of positive cases collected is $X_{k,t}$. For each region $k$, for $k=1,2,\ldots,K$, the
CUSUM statistics $W_{k,t}$ can be computed for each time $t$ and region $k$ can be derived as 
\begin{equation}
W_{k,t}=\max\{W_{k,t-1},0\}+c_{k,t}\log\frac{1-q}{1-p}+X_{k,t}\log\frac{q(1-p)}{p(1-q)}\label{eq-Update-Single}
\end{equation}
\end{prop}
The proof is detailed in Appendix \ref{proof-prop-CUSUM}. 


Here, we will raise the alarm when the largest of the CUSUM statistics $\max_{k} W_{k,T}$ exceeds some threshold $c$ at some time $T$. Finally, we will discuss the procedure of selecting the threshold $c$ in Section \ref{subsec:Tuning-Parameter-Selection}. 


\subsubsection{Planning Objective Function for Resource Allocation \label{subsubsec:Planning}}

In this Section, we present an efficient algorithm for resource allocation at the next time point
for hotspot detection. Overall, a good resource allocation algorithm for change detection should achieve a good balance
between exploration and exploitation. Here, exploration in the problem formulation implies the test
kits should be distributed to the regions that were not measured before, and exploitation means
that when the change occurs, the planning will focus more on the region with the change for the quickest change detection.
To better tackle the exploration-exploitation trade-off problem, we apply the UCB as the target function to be optimized. \citet{auer2002using} suggested that the UCB criteria for the MAB problem maximize the final reward.
This method takes advantage not only of the mean reward $\mu$ but also the standard deviation of the reward $\sigma$.
In this case, we should not only focus on the region leading to the larger reward but also the action with larger variability, which has the potential to lead to a larger reward. Specifically, in this hotspot detection application, for the true hotspot region with a larger infection rate, we may want to continue to distribute more test kits to the particular region the next day to encourage further exploitation. However, when a region has not been observed enough in previous days, it typically has a larger variance, and some testing resources will be distributed in the regions with large variability the next time to encourage further exploration.

Given that the log-likelihood ratio $\log\frac{f^1_{k}(X_{k,t})}{f^{0}_k(X_{k,t})}$ for change point detection is often a good quantification of how likely the change may happen in a particular region $k$, here we propose to use the sum of UCB of the likelihood ratio in the CUSUM statistics as the reward function. For the convenience of notation, we denote the log-likelihood ratio in the CUSUM statistics in each region $k$ by 
\begin{equation}
\Delta r_{k,t}=c_{k,t}\log\frac{1-q}{1-p}+X_{k,t}\log\frac{q(1-p)}{p(1-q)}\label{eq-Increment-Single}
\end{equation}

Therefore, the UCB function of the summation of the likelihood ratio statistics can be expressed as

\begin{equation}
f(\mathbf{c}_{t})=\sum_{k=1}^{K}\left[E[\Delta r_{k,t}]+\sqrt{Var(\Delta r_{k,t})}\right]\label{eqUCB}
\end{equation}

Here, the goal is to find the best resource allocation planning method to optimize the resource distribution $\mathbf{c}_{t}=(c_{1,t},\ldots,c_{K,t})$ at time $t$. However, one of the challenges is that the $X_{k,t}$ is the number of positive testing results at time $t$, which is a random variable at time $t$. 
In Section \ref{subsubsec:Bayesian},
we will derive the Bayesian updating procedure to update the posterior distribution of $X_{k,t}$, which can be used estimate the $E[\Delta r_{k,t}]$ and $Var(\Delta r_{k,t})$.

\subsubsection{Bayesian Updating of $X_{k,t}$ \label{subsubsec:Bayesian} }

To estimate the distribution of $X_{k,t}$, we assume that it follows a binomial distribution of $\text{Bin}(c_{k,t},p_{k})$, where $p_{k}$ is the true infection rate to be estimated. Therefore, an online estimation of the rate is needed for better planning and resource allocation. Here,  we propose a weighted Bayesian update method to estimate the infection rate in Proposition \ref{prop-posterior}. 
The reason to consider such a weighted updating procedure is that the samples in recent time points are typically more likely to represent the current state of the system and should be used to better detect the change. Therefore, more weight should be put into the recent time points. For simplicity, we propose to use a exponential decayed weight $\lambda_t = w^{T-t}$ for each sample $X_{k,t}$. 

Finally, the posterior distribution  $p_{k,t}$ can be derived as 
\begin{equation}
    P(p_{k}|X_{k,1},\cdots, X_{k,t}) \propto P(p_{k}) \prod_{t=1}^T P(X_{k,t}|p_{k})^{\lambda_t},
\end{equation}  
where  $P(p_{k})$ is the prior distribution of the infection rate, which is assumed to follow a Beta distribution as  $\text{Beta}(a,b)$. 
 $P(p_{k}|X_{k,1},\cdots, X_{k,t})$ is the posterior distribution of $p_k$, which can be derived in Proposition \ref{prop-posterior}. 

\begin{prop}
\label{prop-posterior} Suppose that the data $(c_{k,t}, X_{k,t})$ of the $k^{th}$ region until time $T$ is collected, where $c_{k,t}$ is the number of total tests distributed to region $k$ and time $t$ and $X_{k,t}$ is the daily positive cases. We further assume the prior distribution of $p_{k}$ is $\mathrm{Beta}(a,b)$,
the weighted posterior distribution of $P(p_{k}|X_{k,1},\cdots, X_{k,T})$ can be derived as
\begin{equation}
\mathrm{Beta}(a+\sum_{t=1}^{T}X_{k,t}w^{T-t},b+\sum_{t=1}^{T}(c_{k,t}-X_{k,t})w^{T-t})\triangleq \mathrm{Beta}(\alpha_{k,t},\beta_{k,t})\label{Distribution}
\end{equation}
\end{prop}
The proof is given in Appendix \ref{proof-prop-posterior}. 
Here, the updated posterior distribution of $p_{k}|X_{k,1},\cdots, X_{k,T}$ can be used to calculate the uncertainty in $X_{k,t}$
to derive the UCB function in the planning method.


\subsubsection{Optimization Algorithm for Planning \label{subsec:Optimization-algorithm}}

By plugging in the weighted posterior distribution of  $P(p_{k}|X_{k,1},\cdots, X_{k,T})$  into the UCB reward function in (\ref{eqUCB}), we will derive the objective function to be optimized and propose an integer programming algorithm to optimize the number of test kits $c_{k,T+1}$ in each region $k$ at time $T+1$. 
\begin{prop}
\label{prop-planning-single} Suppose $X_{k,t}\sim \mathrm{Bin}(c_{k,t},p_{k})$ and the posterior distribution  $p_{k}|X_{k,1},\cdots, X_{k,T} \sim \mathrm{Beta}(\alpha_{k,t},\beta_{k,t})$.
The target function $f_t(\mathbf{c}_{T+1})=\sum_{k=1}^{K}\left[E[\Delta r_{k,T+1}]+\sqrt{Var(\Delta r_{k,T+1})}\right]$ can be derived as 
\begin{footnotesize}
\begin{align}
f_t(\mathbf{c}_{T+1}) &= \sum_{k=1}^{K} f_{k,t} (c_{k,T+1})\nonumber\\ &=\sum_{k=1}^{K}\left[\frac{\alpha_{k,t}}{\alpha_{k,t}+\beta_{k,t}} c_{k,T+1}+\sqrt{c_{k,T+1}\frac{\alpha_{k,t}\beta_{k,t}}{(\alpha_{k,t}+\beta_{k,t})(\alpha_{k,t}+\beta_{k,t}+1)}(\frac{c_{k,T+1}}{\alpha_{k,t}+\beta_{k,t}}+1)}\right]\label{eq-Target-Independent}
\end{align}
\end{footnotesize}

\end{prop}
The proof is shown in Appendix \ref{proof-prop-planning-single}. 
Therefore, Algorithm \ref{alg:planning} is proposed to determine the best resource allocation $(c_{1,T+1},\ldots,c_{K,T+1})$ at time $T+1$ as follows: 

\begin{align}
\hat{c}_{1,T+1},\ldots,\hat{c}_{K,T+1} & =\arg\max f_t(\mathbf{c}_{T+1})\label{eq-Planning-Independent}\\
 & s.~t.~\sum_{k=1}^{K}c_{k,T+1}\le C,c_{k,T+1}\ge0,k=1,\ldots,K\nonumber 
\end{align}

\begin{figure}
    \centering
    \includegraphics[scale=0.5]{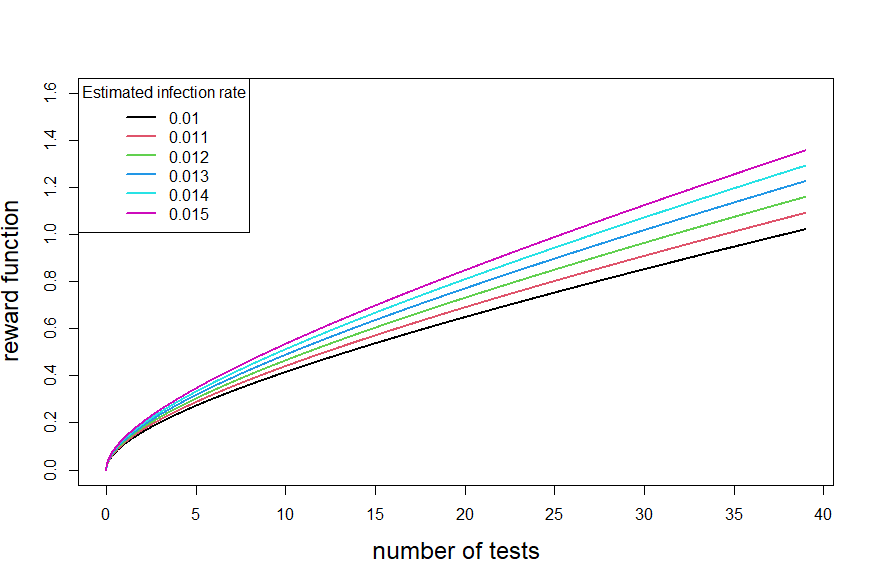}
    \caption{Example of reward functions with different estimated infection rate}
    \label{fig: reward example}
\end{figure}
Here, we would like to give a brief discussion on the reward function $f_{k,t} (c_k)$ defined in \eqref{eq-Target-Independent}. An example of the function  $f_{k,t} (c_k) $ with different infection rates are given in Figure \ref{fig: reward example}. 
The expectation part and standard deviation part is larger for regions with a higher estimated infection rate. However, there is a diminishing return effect on this reward function, which implies that as more testing resources are distributed in the highly infected regions, the increase of the reward function will become smaller.  At a certain point, it may be less rewarding to distribute additional tests to the highly infected region compared to the low infected region. To demonstrate, Figure \ref{fig: reward example} shows that $f_5(x)>f_k(x)$ for $k=1,2,3,4$ and $x\in[1,50]$, but $f_k'(x)>f_5'(y)$ for some $k=1,2,3,4$  and $x,y\in[1,50]$.

To optimize problem \eqref{eq-Planning-Independent}, a greedy algorithm is proposed. The idea of the greedy algorithm is to distribute the testing kits one by one to the region that can result in the largest target increment until all the tests are properly distributed. Given the concavity of $f_{t}(\mathbf{c})$, such a greedy approach can guarantee the global optimum. This property will be discussed in Section \ref{prop-convex}. The greedy algorithm is designed as Algorithm \ref{alg:planning}.

\LinesNumberedHidden{
\begin{algorithm}[t]
\caption{Optimal Resources Planning by the Greedy Algorithm}
\label{alg:planning}
\SetAlgoLined
\Parameter{$\alpha_{k,t}$, $\beta_{k,t}$, $k=1,\ldots,K$}
\KwResult{$c_{k}$, $k=1,\ldots,K$}
\textbf{Step 1:} Initialize $c_{k}=0$ for $k=1,\ldots,K$\\
\textbf{Step 2:} \For{$i=1$ \KwTo $C$}{
    \textbf{Step 2.1:} \For{$k=1$ \KwTo $K$}{
Estimate the increment in region $k$: $\Delta_k=f_k(c_k+1)-f_k(c_k)$.
}
\textbf{Step 2.2:} Let $k_m=\arg\max_k \Delta_k$. $c_{k_m}=c_{k_m}+1$
}
\end{algorithm}}



Finally, the entire process can be  shown in Algorithm 2. Overall, the steps of the algorithm are described here: 1) On day one, we have an initial test distribution. If the prior distribution is set the same for all regions, uniform distribution can be used. The confirmed number of $X_{k,1}$ will be collected on Day 1. 2) The collected dataset $X_{k,t}$ can be used to update the posterior distribution of the rate $p_k$ by \eqref{Distribution}. 3) Such distribution information can be plugged into \eqref{eq-Planning-Independent} to optimize the test distribution of the next day. Such a procedure can be estimated recursively until the current time $T$. 

\LinesNumberedHidden{
\begin{algorithm}
\caption{Proposed ARA-CUSUM Algorithm}
\label{alg: process}
\SetAlgoLined 
\Parameter{$a$, $b$, $w$, $c$} 
\KwResult{$W_{k,t}$,$k=1,\ldots,K$,$t=1,2,\ldots$}
\textbf{Step 1:} Initialize $c_{k,1}$, $X_{k,1}$ and $W_{k,1}$ for $k=1,\ldots,k$\; 
\textbf{Step 2:} \For{$t=1,2,\ldots$}{ 
\textbf{Step 2.1: Bayesian Update:}
Set $\alpha_{k,t}=a+\sum_{i=1}^{t}X_{k,i}w^{t-i}$; $\beta_{k,t}=b+\sum_{i=1}^{t}(c_{k,i}-X_{k,i})w^{t-i}$\; 
\textbf{Step 2.2: Resource Allocation:} Solve   $c_{k,t+1}$ in the optimization problem  \eqref{eq-Planning-Independent} by Algorithm \ref{alg:planning} \; 
\textbf{Step 2.3: Monitoring Statistics Update and Detection:} Update $W_{k,t+1}$ by (\ref{eq-Update-Single}) \; 
\textbf{Step 2.4:} \uIf{$W_{k,t+1}>c$
for some $k$}{ Raise alarm at time $t+1$ and county $k$\; } }
\end{algorithm}}


\subsubsection{Property of the Optimization Algorithm \label{subsec:Property}}
In this subsection, we would like to prove that the greedy approach proposed in Section \ref{subsec:Optimization-algorithm} can obtain the global optimum. 

\begin{lemma}\label{lemma-concavity}
The target function \eqref{eq-Target-Independent} is the sum of $K$ concave functions and is therefore concave.
\end{lemma}
The target function in each region can be expressed as $f_{k,t}(x)=a x+\sqrt{b x+c x^2}$. One can easily verify that the second order derivative of $f_{k,t}(\cdot)$ is negative when $b,c>0$.

\begin{prop}\label{prop-convex}
The algorithm \ref{alg:planning} can obtain the global optimum for the optimization problem \eqref{eq-Planning-Independent}.
\end{prop}
\begin{proof}
From Lemma \ref{lemma-concavity}, $f_{k,t}(\cdot)$ is an increasing and strictly concave function. For any fixed time $t$, we denote $g_k(x)=f_{k,t}(x)-f_{k,t}(x-1)$ be the increment of the reward function to distribute one more tests on region $k$. $g_k(\cdot)$ will be a positive decreasing function for any $k$. Then the total reward function given the test distribution $c_{1,t},\ldots,c_{K,t}$ is 
$$
f_{t}(c_{1,t},\ldots,c_{K,t})=\sum_{k=1}^K \left(\sum_{x_k=1}^{c_{k,t}}g_k(x_k)\right)
$$

Now consider the increment parameter matrix
$$
\begin{pmatrix}
g_1(1) & g_1(2) & \cdots & g_1(C)\\
g_2(1) & g_2(2) & \cdots & g_2(C)\\
\vdots & \vdots & \ddots & \vdots \\
g_K(1) & g_K(2) & \cdots & g_K(C)
\end{pmatrix}
$$

In Algorithm \ref{alg:planning}, the first iteration will choose the largest value in the first column. Since each row in the matrix is decreasing, the first iteration will also choose the largest value in the matrix. Each of the remaining iterations will choose the largest value in the matrix that has not been chosen. Therefore, Algorithm \ref{alg:planning} chooses the top $C$ largest values in the matrix.  And because the target function can be written by the sum of $C$ values in the matrix, Algorithm \ref{alg:planning} gives the optimal solution to the optimization problem \eqref{eq-Planning-Independent}

\end{proof}

\subsection{Tuning Parameter Selection \label{subsec:Tuning-Parameter-Selection}}

There are several tuning parameters that need to be determined for the algorithm. The key parameters are prior distribution parameter $a,b$, weight $w$ and threshold $c$. In this Section, we will discuss the effect of these tuning parameters and how to decide them.
\begin{itemize}
\item \textbf{Prior Distribution}: There are two parameters $a,b$ in the prior distribution. The prior distribution reflects the prior knowledge about the infection rate. Given that the infection rate for the null hypothesis is $p$, we can also set the ratio to be $p$, i.e., $\frac{a}{a+b}=p$, assuming no change happens. Since $p$ is a preset constant, we only need to determine $a$. When $a$ is large, the influence of prior distribution is large, and the difference in the infection rate estimation from the posterior distribution will be smaller. Thus, the difference in test distribution will not be very large. In the simulation, we will choose $a+b=0.5\times C$. 
\item \textbf{Weight $w$}: The weight is used in Bayesian update to get the posterior distribution of $X_{k,t}$. Weight allows the algorithm to take not only the information from one day before but also the previous days. A large weight will make the test distribution in each region smooth across different days. Therefore, when occasionally the observed infection rate is high, the algorithm will only distribute a reasonably larger amount of testing kits on the next day rather than some extreme testing kit distribution. In the simulation, we choose $w=0.3$. 
    \item \textbf{Total Resources $C$}: This parameter determines the total number of tests to be distributed in all regions. Typically, $C$ is pre-determined by the resource constraint. 
    \item \textbf{Threshold $c$}: This parameter determines whether an alarm should be raised in each region every day. Higher threshold $c$ leads to high in-control average run length ($ARL_0$) and higher out-of-control $ARL_1$. Typically, we would like to maintain a preset $ARL_0$, which determines the threshold $c$. Overall, for a preset $C, p, q$, a simulation study can be used to determine the best $C$ with certain $p$ and $q$. 
    We will determine this parameter by the simulation work and use a preset average run length. 
\end{itemize}


\section{Simulation \label{sec:Simulation}}

In this Section, we will start with the simulation setup in Section \ref{subsec:simulation setup} and
then discuss the result evaluation in Section \ref{subsec:result}. The balance of exploration and exploitation is discussed in Section \ref{subsec: exploration and exploitation}. 

\subsection{Simulation Setup \label{subsec:simulation setup}}

In this Section, we will evaluate the proposed methodology using a simulation study. To generate the simulation dataset, we will simulate a true infection rate for each county in the WA as an example. 

There are $39$ counties in WA, i.e., $K=39$.
Before the distribution change, the true infection proportion is set to be $p=0.01$ for all counties. We will compare the detection delay for different infection proportions where the out-of-control infection rates are $q=0.025,0.03,0.04,0.05$ in the first county after the distribution change. We further assume that the total number of tests is $3900$, where each county takes an average of $100$ tests per day. The weighted update parameter $w$ is set to be $0.3$ and the prior distribution parameter $a,b$ are set as the $50\%$ of the total tests to make the prior distribution as 
$a=3900\times50\%\times0.01=19.5$, $b=3900\times50\%\times0.99=1930.5$.

The process of the simulation is described as follows: 
\begin{enumerate}
\item The test distribution $c_{k,t}$ in each county determined from the previous iteration can be used.
\item Before the change point $t<t_0$, we will sample the observed positive cases by $X_{k,t}\sim \text{Bin}(c_{k,t},p)$.
\item After the change point $t>t_0$, we will sample the observed positive cases in the first county by the out-of-control infection rate $q$ and $X_{1,t}\sim \text{Bin}(c_{1,t},q)$. We will sample the rest counties still by $X_{k,t}\sim \text{Bin}(c_{k,t},p)$
\end{enumerate}

Finally, we will follow Algorithm \ref{alg:planning} for the proposed algorithm. The CUSUM statistics are updated, and the test distribution of the next day is determined by the proposed method or benchmark methods. When studying the average run length and detection delay, we will stop the process when the alarm is raised. 

In this simulation study, we will compare the proposed algorithm with the following two benchmark methods. 
\begin{enumerate}
\item \textbf{Even Distribution}: The first benchmark (denoted as Even) is a simple approach, which always evenly distributes the tests to all the counties. This even distribution will focus on exploration of the testing resources but do not have the power to exploit for a specific region. 
\item \textbf{Top-R Distribution}: The second benchmark (denoted as top-R') is inspired by the Top-R strategy proposed in \cite{liu2015adaptive}. We first evenly distribute all tests evenly into $20$ batches. For in total $3900$ testing kits, it averages out $195$ testing kits per batch. We then select the top $20$ regions with the largest CUSUM statistics in the previous day. Here, we select the top $20$ regions, given it is a little more than half of the whole regions.
\end{enumerate}

\subsection{Evaluation of the Detection Delay and Precition \label{subsec:result}}


To compare this algorithm with the benchmarks, we will first need to determine the thresholds for the proposed method and benchmark methods. Here, binary search can be used to decide the optimal threshold for each benchmark method to maintain the in-control average run length (ARL) as 200 for all benchmark methods with 1000 replications.

To evaluate the proposed methods, we use the following three metrics: 1) the out-of-control average run length ($\mathrm{ARL}_1$): $\mathrm{ARL_1}$ is defined as the average detection delay after the change has happened. Overall, we aim to reduce the ARL when keeping the same $ARL_0 = 200$; 2) Detection precision (DP): DP is evaluated as the proportion in the $1000$ iterations that the alarm is raised in the correct county. Overall, we would aim to increase the DP so that we can identify the county that is accountable for the change more accurately.
3) the standard deviation of run-length (SDRL): SDRL is evaluated by computing the standard deviation with $1000$ iterations. Overall, a smaller SDRL would imply a more robust detection with smaller variability for detection delay.
These three metrics for the proposed method and the benchmark methods (namely the evenly distributed and the Top-R distributed) are evaluated in Table \ref{table-result-1}. 

From Table \ref{table-result-1}, we can conclude that ARL is the smallest for the different $q$ values. Especially, when $q$ is small (i.e., 0.025 and 0.03), the detection delay can even be greatly reduced (i.e., 7.893 for "proposed", 14.885 for "Even", and 10.385 for "Top-R" for $q=0.025$). For larger $q$, the differences between the proposed methods and the second best method, such as "Top-R" are similar, given it becomes easy to detect all the methods. However, the proposed method still outperforms others in this case. Finally, we find that the SDRL of the proposed method is also the smallest, especially when $q$ is small. This implies that the proposed methods can also make the detection much more robust by reducing the standard deviation of the detection delay.

\begin{table} 
\caption{ The result of detection delay and detection precision when $ARL_0 \approx 200$}
\centering 
\begin{tabular}{|c|c|c|c|c|c|} 
\hline 
Metric & Methods  & $q=0.025$ & $q=0.03$ & $q=0.04$ & $q=0.05$\tabularnewline
\hline 
\multirow{3}{*}{$\mathrm{ARL_1}$} & Proposed & $\mathbf{7.893}$ & $\mathbf{4.958}$ & $\mathbf{3.388}$ & $\mathbf{2.863}$\tabularnewline
\cline{2-6} 
 & Even  & $14.885$ & $7.93$ & $4.408$ & $3.299$\tabularnewline
\cline{2-6}
& Top-R & $10.385$ & $5.688$ & $3.443$ & $2.891$\tabularnewline
\hline 
\multirow{3}{*}{DP} & Proposed & $\mathbf{0.918}$ & $\mathbf{0.932}$ & $0.938$ & $0.939$\tabularnewline
\cline{2-6} 
 & Even & $0.892$ & $0.926$ & $\mathbf{0.939}$ & $\mathbf{0.943}$\tabularnewline
\cline{2-6}
& Top-R & $0.904$ & $0.918$ & $0.936$ & $0.937$\tabularnewline
\hline
\multirow{3}{*}{SDRL} & Proposed & $\mathbf{4.67}$ & $\mathbf{2.47}$ & $\mathbf{1.31}$ & $\mathbf{1.01}$\tabularnewline
\cline{2-6} 
 & Even  & $12.81$ & $5.48$ & $2.28$ & $1.45$\tabularnewline
\cline{2-6}
& Top-R & $10.556$ & $3.72$ & $1.63$ & $1.12$\tabularnewline
\hline 
\end{tabular}
\label{table-result-1}
\end{table}

\subsection{Exploration and Exploitation}\label{subsec: exploration and exploitation}

As formerly mentioned, a good change point detection with a sampling control algorithm should balance both exploration and exploitation behavior. In this simulation, exploration means that we will need to explore all counties, given that any county can undergo large changes at certain times. Exploitation means that we will need to distribute more resources to the counties that have been detected or have the potential to be the hotspot. Figure \ref{Simu_Distribution} and Figure \ref{Simu_County} give a more intuitive understanding of the exploration and exploitation behavior of this proposed algorithm from the following simulation study setup. In the first $500$ days, all the counties are in control with the same infection rate $p=0.01$. From day $501$ to day $1000$, only the first county, Adams County is out of control with the infection rate $q=0.05$, while the rest counties stay the same with $p=0.01$. Figure \ref{Simu_Distribution} shows the density and the boxplot of the number of tests in control.  Figure \ref{Simu_County} shows the median number of test kits in each county under different conditions, in control and out of control.

Here are some behaviors of the algorithm. 
When there is no distribution change, the distributions of the number of tests in each county are similar, see Figure \ref{Simu_Distribution} and Figure \ref{Simu_County_IC}. This shows that the proposed method is able to explore all the counties undergoing potential changes.
However, if a county has a higher infection rate on the current day, it is likely to have higher tests the next day. If the sampled positive proportion is small, the increment of CUSUM statistics will be likely to be $0$. If the distribution change happens in that county, the CUSUM test statistic will be likely to grow dramatically and we will continue to focus on that county further. See Figure \ref{Simu_County_OC}, when Adams County is out of control, the median number of the distributed tests of Adams County is significantly larger than the number of other counties. In this case, the CUSUM test statistic of that Adams County is more likely to increase dramatically, which is a good behavior of exploitation.

\begin{figure}[H]
    \begin{subfigure}{.5\textwidth}
        \centering
        \includegraphics[width=\linewidth]{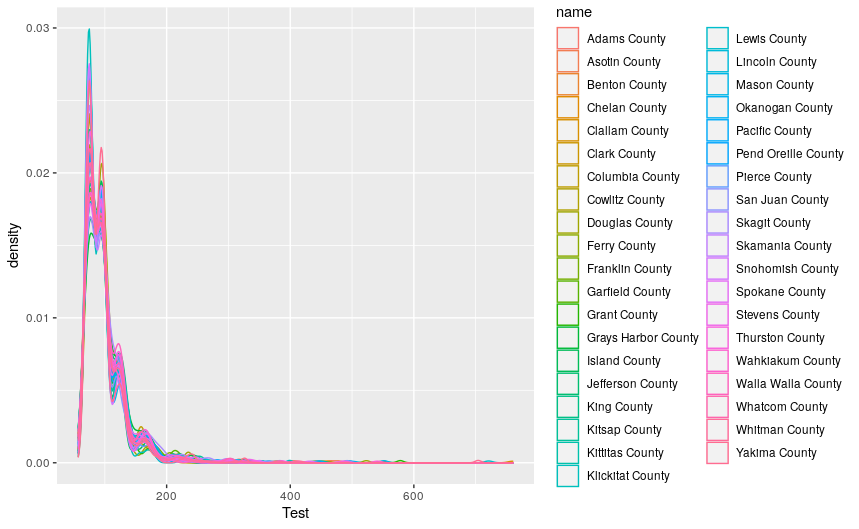}  
        \caption{Test Distributions Density In Control}
        \label{Simu_Distribution_IC}
    \end{subfigure}
    \begin{subfigure}{.5\textwidth}
        \centering
        \includegraphics[width=\linewidth]{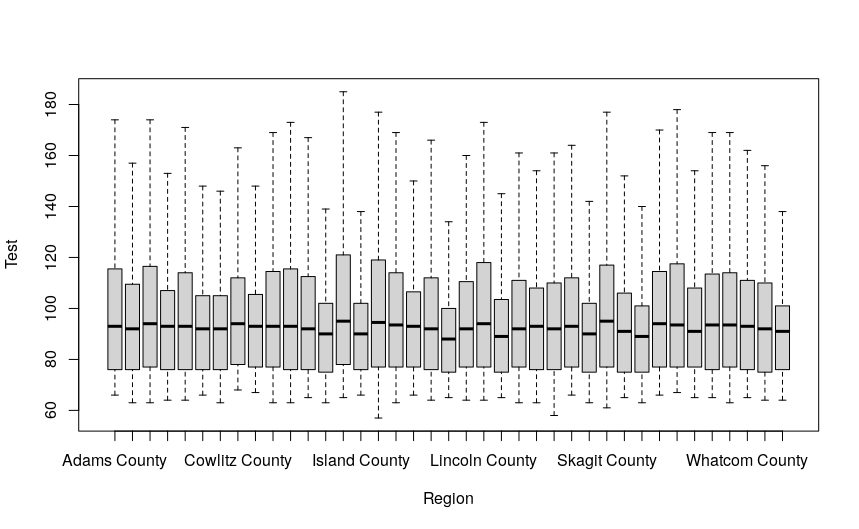}  
        \caption{Boxplot of Number of Tests In Control}
        \label{Simu_boxplot_IC}
    \end{subfigure}

        \caption{Visualization of Test Distributions In Control}
    \label{Simu_Distribution}
\end{figure}

\begin{figure}[H]
    \begin{subfigure}{.5\textwidth}
        \centering
        \includegraphics[width=\linewidth]{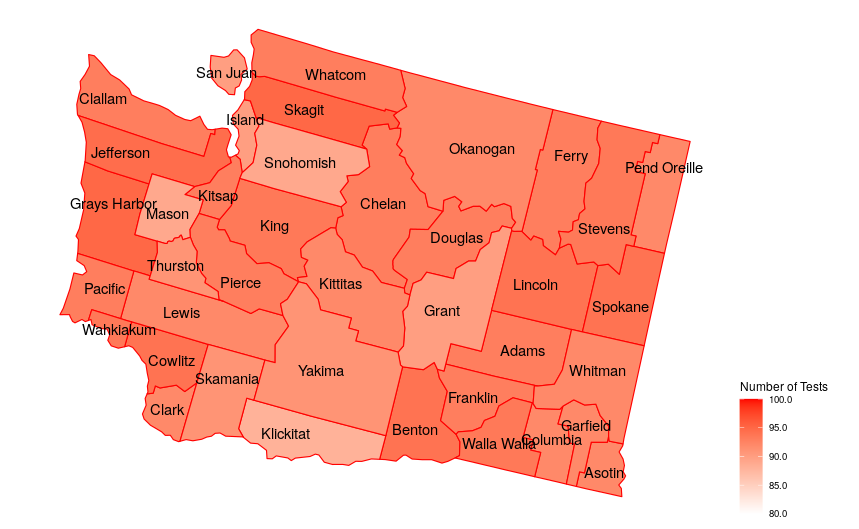}  
        \caption{Median Number of Tests In Control}
        \label{Simu_County_IC}
    \end{subfigure}
    \begin{subfigure}{.5\textwidth}
        \centering
        \includegraphics[width=\linewidth]{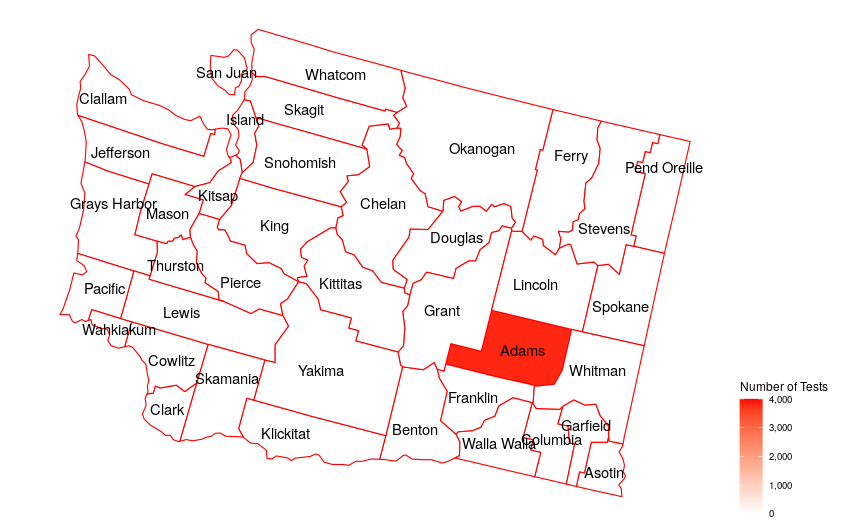}  
        \caption{Median Number of Tests Out of Control}
        \label{Simu_County_OC}
    \end{subfigure}
    \caption{Median Number of Tests}
    \label{Simu_County}
\end{figure}

\section{Case Study}\label{sec:Case Study}

In this case study, we would like to evaluate the proposed methods to the real case study. The data we are working on is the county-level
daily positive cases, and the parameters are $K=39$, $p=0.01$, $q=0.05$, $a=19.5$, $b=1930.5$ and
$w=0.3$. The average number of tests in each county is $100$.

\subsection{Detection of the Out-of-control County}
The alarms were raised in Yakima county on Jun 19, 2020. The comparison figure of the test statistics is in Figure \ref{Test-Statistics}. Starting from day $135$, the infection rate of Yakima county started to show a dominating trend over the other counties. Therefore, the algorithm emphasizes on Yakima county to seek for potential risk. But the true infection rate is not large enough to produce a positive increment to the CUSUM statistics. Thus, we will see this oscillating trend. When the infection rate of Yakima county increases consistently, the CUSUM statistics will increase dramatically. Finally on June 21, 2021, the CUSUM statistics in Yakima County exceeds the threshold of $6.5$ and the alarm is raised. 

\begin{figure}[H]
    \centering
    \includegraphics[width=0.6\linewidth]{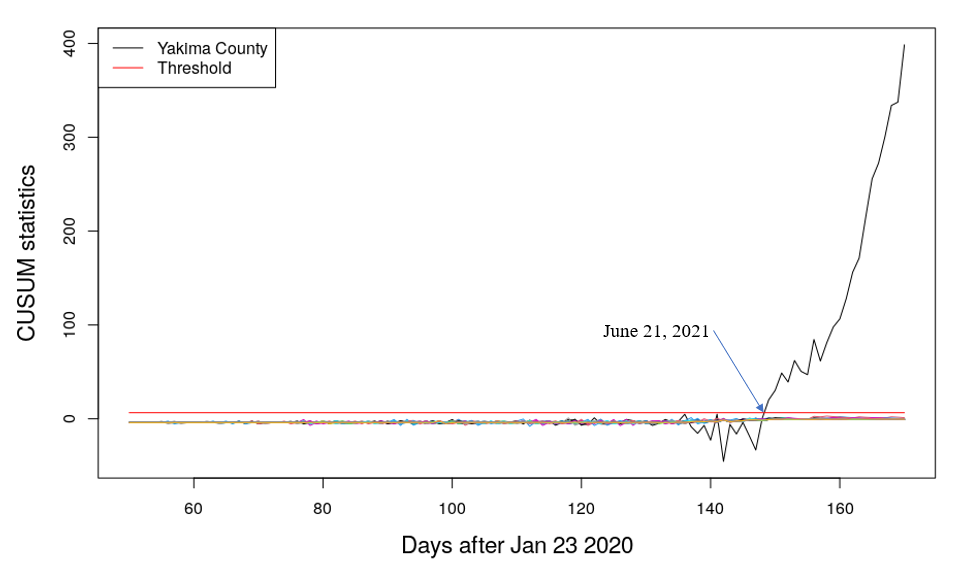}
    \caption{Comparison of the Test Statistics}
    \label{Test-Statistics}
\end{figure}


\subsection{Behavior of Sampling Resource Distribution}

The plots in Figure~\ref{Test-Distribution-Case} are examples of the test distribution of this case study before and after the alarm is raised. This shows similar behavior as the simulation study. The test distribution before the alarm is close to an even distribution to explore all the counties for the risks. The test distribution after the alarm focuses on  Yakima county and consistently monitors this county to show good behavior of exploitation.

\begin{figure}[H]
    \begin{subfigure}{0.5\textwidth}
        \centering
        \includegraphics[width=\linewidth]{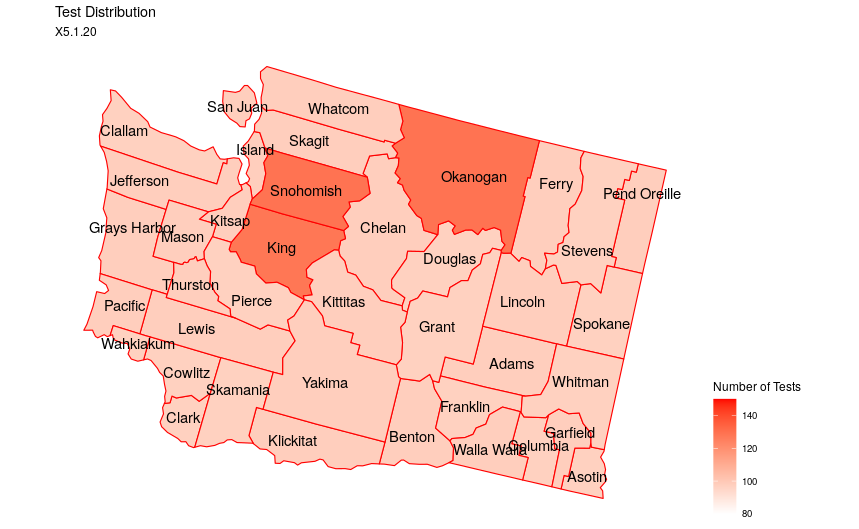}
        \caption{Example of test distribution in control}
        \label{Test-Case-IC}
    \end{subfigure}
    \begin{subfigure}{0.5\textwidth}
        \centering
        \includegraphics[width=\linewidth]{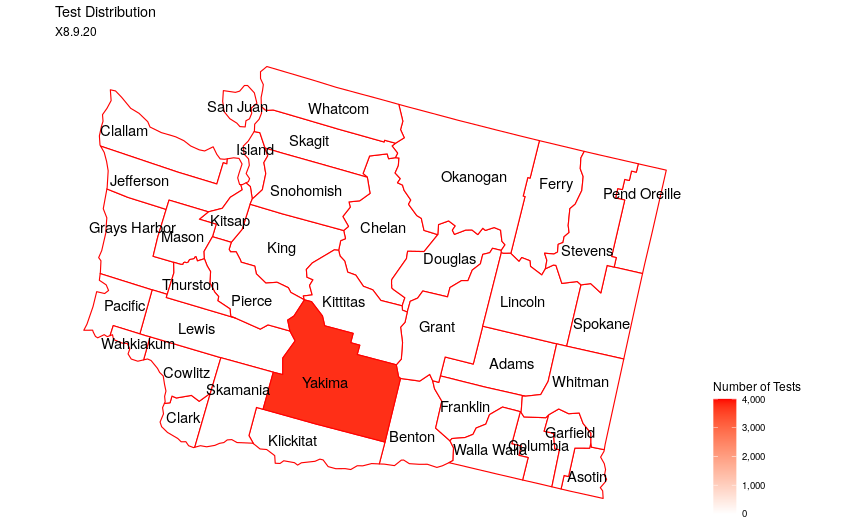}
        \caption{Example of test distribution out of control}
        \label{Test-Case-OC}
    \end{subfigure}
    \caption{Example of test distribution in WA}
    \label{Test-Distribution-Case}
\end{figure}


\section{Conclusion}\label{sec:Conclusion}
In conclusion, the proposed method is able to handle the challenge of exploration and exploitation trade-offs and limited resources. In the simulation work, we show that the adaptive sampling method can beat the two benchmarks either in the sense of detection delay or detection precision. In the case study in WA, the alarm was raised in Yakima county. This algorithm works well when there is only one out-of-control region. When there are multiple out-of-control regions, this algorithm will focus on the first region with a higher out-of-control risk. If there are multiple potential regions at risk, one can start the algorithm from the beginning again when an alarm is raised and remove the alerted region. In the future, we can address this concern and detect multiple out-of-control regions in a time series manner. 
Some other future work includes taking geometric distance and graphing into consideration and solving the detection problem when there are multiple out-of-control regions.

\bibliographystyle{plain}
\bibliography{references}

\appendix

\section{Proof of Proposition \ref{prop-CUSUM}} \label{proof-prop-CUSUM}
\begin{proof}
The distribution of $X_{k,t}$ is a binomial distribution with total counts $c_{k,t}$. In the null hypothesis
and alternative hypothesis, the probability is $p$ and $q$, respectively. Thus 
\[
f_{0}(X_{k,t})=\binom{c_{k,t}}{X_{k,t}}p^{X_{k,t}}(1-p)^{c_{k,t}-X_{k,t}}
\]
\[
f_{k}(X_{k,t})=\binom{c_{k,t}}{X_{k,t}}q^{X_{k,t}}(1-q)^{c_{k,t}-X_{k,t}}
\]
Plugging in the formula above to equation (\ref{eq-CUSUM}), we have 
\begin{align*}
\log\frac{f_{k}(X_{k,t})}{f_{0}(X_{k,t})} & =X_{k,t}\log\frac{q}{p}+(c_{k,t}-X_{k,t})\log\frac{1-q}{1-p}\\
 & =c_{k,t}\log\frac{1-q}{1-p}+X_{k,t}\log\frac{q(1-p)}{p(1-q)}
\end{align*}
And we have the result of equation (\ref{eq-Update-Single}). 
 
\end{proof}

\section{Proof of Proposition \ref{prop-posterior}} \label{proof-prop-posterior}
\begin{proof}
For binomial distribution, a reasonable conjugate prior distribution is the beta distribution. The posterior
distribution from a weighted update is 
\[
f(p_{k,t}|X_{k,1},\ldots,X_{k,T})\propto f_{0}(p_{k,t})\prod_{t=1}^{T}f(X_{k,t}|p_{k,t})^{w_{t}^{T}}
\]
where $w_{t}^{T}$ is a series of time decayed weights with increasing order. Here we set it to be $w_{t}^{T}=w^{T-t}$.
Recall that $X_{k,t}\sim binomial(c_{k,t},p_{k,t})$, we have 
\begin{align*}
f(p_{k,t}|X_{k,1},\ldots,X_{k,T}) & \propto f_{0}(p_{k,t})\prod_{t=1}^{T}f(X_{k,t}|p_{k,t})^{w_{t}^{T}}\\
 & \propto p_{k,t}^{a-1}(1-p_{k,t})^{b-1}\prod_{t=1}^{T}\left(p_{k,t}^{X_{k,t}}(1-p_{k,t})^{c_{k,t}-X_{k,t}}\right)^{w_{t}^{T}}\\
 & =p_{k,t}^{a+\sum_{t=1}^{T}w_{t}^{T}X_{k,t}-1}(1-p_{k,t})^{b+\sum_{t=1}^{T}w_{t}^{T}(c_{k,t}-X_{k,t})-1}
\end{align*}
This is the kernel of the distribution $Beta(a+\sum_{t=1}^{T}w^{T-t}X_{k,t},b+\sum_{t=1}^{T}w^{T-t}(c_{k,t}-X_{k,t}))$ 
\end{proof}

\section{Proof of Proposition \ref{prop-planning-single}} \label{proof-prop-planning-single}
\begin{proof} 
We refer to equation (\ref{eq-Increment-Single}) to get $\mu+\sigma$ of the increment. Recall that
$X_{k,t}\sim Binomial(c_{k,t},p_{k,t})$, $p_{k,t}\sim Beta(\alpha_{k,t},\beta_{k,t})$ and $c_{k,t}$
are the decision variables, which can be regarded as constant. We have 
\begin{align*}
E\Delta r_{k,t} & =c_{k,t}\log\frac{1-q}{1-p}+\log\frac{q(1-p)}{p(1-q)}E[E(X_{k,t}|p_{k,t})]\\
 & =c_{k,t}\log\frac{1-q}{1-p}+\log\frac{q(1-p)}{p(1-q)}E[c_{k,t}p_{k,t}]\\
 & =c_{k,t}\log\frac{1-q}{1-p}+\log\frac{q(1-p)}{p(1-q)}\frac{c_{k,t}\alpha_{k,t}}{\alpha_{k,t}+\beta_{k,t}}
\end{align*}
\begin{align*}
Var(X_{k,t}) & =Var(E[X_{k,t}|p_{k,t}])+E[Var(X_{k,t}|p_{k,t})]\\
 & =Var(c_{k,t}p_{k,t})+E(c_{k,t}p_{k,t}(1-p_{k,t}))%
\end{align*}
The first term is 
\begin{align*}
Var(c_{k,t}p_{k,t}) & =c_{k,t}^{2}\frac{\alpha_{k,t}\beta_{k,t}}{(\alpha_{k,t}+\beta_{k,t})^{2}(\alpha_{k,t}+\beta_{k,t}+1)}
\end{align*}
The second term is 
\begin{align*}
E[c_{k,t}p_{k,t}(1-p_{k,t})] & =c_{k,t}E[p_{k,t}]-c_{k,t}E[p_{k,t}^{2}]\\
 & =c_{k,t}\frac{\alpha_{k,t}}{\alpha_{k,t}+\beta_{k,t}}-c_{k,t}\left[\frac{\alpha_{k,t}\beta_{k,t}}{(\alpha_{k,t}+\beta_{k,t})^{2}(\alpha_{k,t}+\beta_{k,t}+1)}+\frac{\alpha_{k,t}^{2}}{(\alpha_{k,t}+\beta_{k,t})^{2}}\right]\\
 & =c_{k,t}\frac{\alpha_{k,t}\beta_{k,t}}{(\alpha_{k,t}+\beta_{k,t})(\alpha_{k,t}+\beta_{k,t}+1)}
\end{align*}
Then, we can calculate the variance of $\Delta r_{k,t}$ as 
\begin{align*}
\sigma^{2} & =\left(\log\frac{q(1-p)}{p(1-q)}\right)^{2}Var(X_{k,t})\\
 & =\left(\log\frac{q(1-p)}{p(1-q)}\right)^{2}c_{k,t}\frac{\alpha_{k,t}\beta_{k,t}}{(\alpha_{k,t}+\beta_{k,t})(\alpha_{k,t}+\beta_{k,t}+1)}(\frac{c_{k,t}}{\alpha_{k,t}+\beta_{k,t}}+1)
\end{align*}
Recall that the target function is formulated to be the sum of the mean and standard deviation. We also
have a constrain that $\sum_{k=1}^{K}c_{k,t}=C$. Then when we sum up $E\Delta r_{k,t}$, the first
term of the expectation part will be a constant and can be left out from the target function. We also
notice that when taking the squared root of $\sigma^{2}$, the coefficient will be the same as the coefficient
of the second term in $E\Delta r_{k,t}$. Thus, this coefficient can also be removed. The simplified
reward function can now be formulated as a formula (\ref{eq-Target-Independent}) and the tests distribution
of next day $t+1$ is to solve the optimization problem (\ref{eq-Planning-Independent}) 
\end{proof}

\section{Sensitivity Analysis} \label{sensitivity}


In this section, we will study the effect of hyperparameters $w$ and $a, b$ on the average run length by performing an additional simulation study. Here, we will start with the effect of $w$ on the average run length. When $a=39/2$, $b=39/2*99$, the average run length is around $200$, when $w=0.01,0.3,0.6$. The detection delay and detection precision are in Table~\ref{tab:comparison}. 

    \begin{table}[h]
        \centering
        \caption{Comparison of detection delay ($\mathrm{ARL_1}$) and Detection precision (DP) for different weight $w$}        
        \begin{tabular}{|c|c|c|c|c|c|}
        \hline
            & $w$ & $q=0.025$ & $q=0.03$ & $q=0.04$ & $q=0.05$\tabularnewline 
        \hline
        \multirow{3}{*}{$\mathrm{ARL_1}$} &
                  $0.01$ & $9.362$ & $5.23$ & $3.391$ & $\mathbf{2.843}$\tabularnewline
                  \cline{2-6}
                  &$0.3$ & $7.893$ & $4.958$ & $\mathbf{3.388}$ & $2.863$\tabularnewline
                  \cline{2-6}
                  & $0.6$ & $\mathbf{7.293}$ & $\mathbf{4.793}$ & $3.438$ & $2.88$\tabularnewline
                  \hline
        \multirow{3}{*}{DP} &
                  $0.01$ & $0.907$ & $0.919$& $ 0.928$ & $0.931$\tabularnewline
                  \cline{2-6}
                  &$0.3$ & $0.918$ & $\mathbf{0.932}$ & $\mathbf{0.938}$ & $\mathbf{0.939}$\tabularnewline
                  \cline{2-6}
                  & $0.6$ & $\mathbf{0.919}$& $0.929$& $0.937 $ & $0.939$\tabularnewline
                  \hline
        \end{tabular}
        \label{tab:comparison}
    \end{table}
    
From Table \ref{tab:comparison}, we can conclude that when $q$ is small, a smaller $w$ results a larger $\mathrm{ARL_1}$ or higher detection delay. However, $w=0.3$ and $w=0.6$ yield similar $\mathrm{ARL_1}$.
When studying one specific iteration, we found that the hyperparameter $w$ affects the distribution of the testing kits. 
Figure \ref{fig:first region} is the box-plot of a region for different weights as $w=0.01, 0.3, 0.6$. We can see that the median of test distributions is similar for different weights. However, when $w$ is large, it results in more days with extremely large numbers of testing kits. This is because when $w$ is large, it puts more weight on the previous days to update the posterior distribution of the infection rate. Therefore, it is harder to reduce the number of testing kits the next day compared to when $w$ is small. 
    
    \begin{figure}[t]
        \centering
        \includegraphics[width=0.8\linewidth]{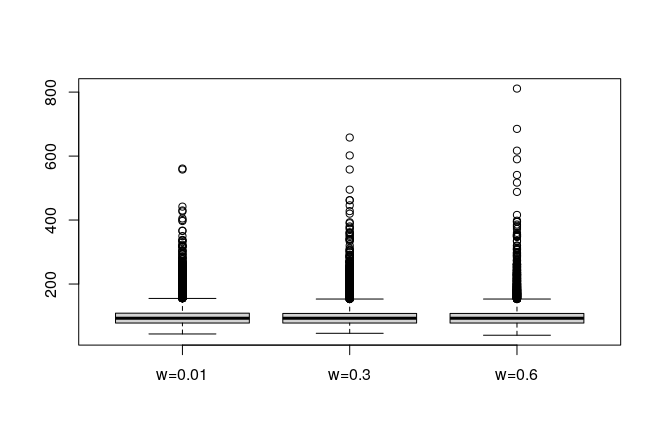}
        \caption{Boxplot of the testing kit distribution in one county before the change for different $w$}
        \label{fig:first region}
    \end{figure}
    \begin{figure}[h]
        \centering
        \includegraphics[width=0.8\linewidth]{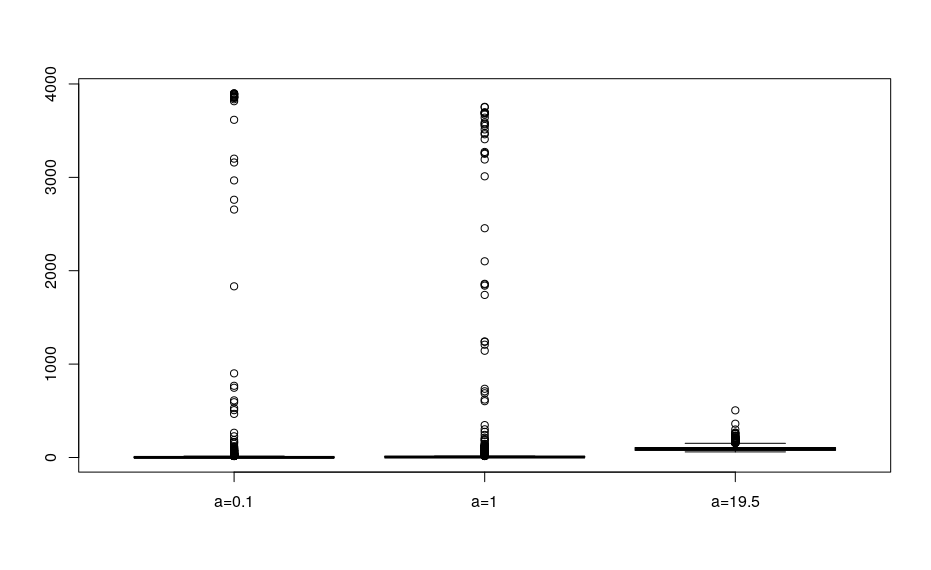}
        \caption{Boxplot of the testing kit distribution in one county before the change for different $a$}
        \label{fig:tuning a}
    \end{figure}

We also studied the impact of different $a$ in the prior distribution, given that the ratio of $a$ and $b$ are kept the same in the simulation study. The testing kit distribution of one county in the in-control state is shown in Figure \ref{fig:tuning a}. 
Intuitively, a higher $a$ makes it harder to distribute more tests in a region. For example, if $a$ is 19.5, the algorithm will allocate fewer tests in the region with $10$ tests and $2$ positive cases than when $a$ is $1$. The simulated result of detection delay and detection precision when the average run length in control is around 200 is as Table \ref{tab:comparison1}. We can see that when $a$ increases from $0.1$ to $19.5$, the detection delay decreases, and the detection precision increases.
Therefore, the choice of $a$ and $b$ depends on how much weight we would like to rely on the prior knowledge.
\begin{table}[h]
        \centering
        \caption{Comparison of detection delay ($\mathrm{ARL_1}$) and Detection precision (DP) for different $a$}        
        \begin{tabular}{|c|c|c|c|c|c|}
        \hline
            & $a$ & $q=0.025$ & $q=0.03$ & $q=0.04$ & $q=0.05$\tabularnewline 
        \hline
        \multirow{3}{*}{$\mathrm{ARL_1}$} &
                  $0.1$ & $17.463$ & $10.229$ & $8.097$ & $7.247$\tabularnewline
                  \cline{2-6}
                  &$1$ & $12.453$ & $8.064$ & $6.214$ & $5.373$\tabularnewline
                  \cline{2-6}
                  & $19.5$ & $7.893$ & $4.958$ & $3.358$ & $2.863$\tabularnewline
                  \hline
        \multirow{3}{*}{DP} &
                  $0.1$ & $0.88$ & $0.91$& $ 0.925$ & $0.93$\tabularnewline
                  \cline{2-6}
                  &$1$ & $0.901$ & $0.923$ & $0.93$ & $0.939$\tabularnewline
                  \cline{2-6}
                  & $19.5$ & $0.918$ & $0.932$ & $0.938$ & $0.939$\tabularnewline
                  \hline
        \end{tabular}
        \label{tab:comparison1}
    \end{table}
    
\end{document}